\documentclass[review,3p,times]{elsarticle}
\usepackage{ecrc}
\usepackage{graphicx}
\usepackage{titlesec}
\usepackage{kpfonts}
\usepackage{epstopdf}
\usepackage{flushend,float,slashbox}
\usepackage{ams,amsmath,amssymb,amsfonts,mathtools,epsfig,caption,subcaption,color}
\volume{00}
\firstpage{1}
\runauth{}

\usepackage[figuresright]{rotating}
\usepackage{bm}
\usepackage{fancyhdr}
\fancyheadoffset[RO,EL]{0pt}
\usepackage{geometry}
\newtheorem{theorem}{Theorem}
\newproof{proof}{Proof}
\begin{document}

\begin{frontmatter}

\dochead{}
\title{
\begin{flushleft}
{\bf \LARGE Estimation and Tracking of AP-diameter of the Inferior Vena Cava in Ultrasound Images Using a Novel Active Circle Algorithm}
\end{flushleft}
}

\author[]{\bf \Large \leftline {Ebrahim Karami$^{*a}$, Mohamed Shehata$^a$, Andrew Smith$^b$}}

\address{\bf  \leftline {$^a$Department of Engineering and Applied Sciences, Memorial University, Canada, }

\bf  \leftline {$^b$Faculty of Medicine, Memorial University, Canada}
}

\cortext[]{Ebrahim karami is responsible for all correspondance (email:ekarami@mun.ca).}

\begin{abstract}
Medical research suggests that the anterior-posterior (AP)-diameter of the inferior vena cava (IVC) and its associated temporal variation as imaged by bedside ultrasound is useful in guiding fluid resuscitation of the critically-ill patient. Unfortunately, indistinct edges and gaps in vessel walls are frequently present which impede accurate estimation of the IVC AP-diameter for both human operators and segmentation algorithms.\\
The majority of research involving use of the IVC to guide fluid resuscitation involves manual measurement of the maximum and minimum AP-diameter as it varies over time. This effort proposes using a time-varying circle fitted inside the typically ellipsoid IVC as an efficient, consistent and novel approach to tracking and approximating the AP-diameter even in the context of poor image quality. In this active-circle algorithm, a novel evolution functional is proposed and shown to be a useful tool for ultrasound image processing. The proposed algorithm is compared with an expert manual measurement, and state-of-the-art relevant algorithms.  It is shown that the algorithm outperforms other techniques and performs very close to manual measurement.
\end{abstract}


\end{frontmatter}

\section{Introduction}
Trauma patients suffering from hemorrhagic shock as a result of blood loss or patients with shortness of breath from volume overload in the setting of congestive heart failure frequently require immediate resuscitation. Fast and accurate assessment of circulating blood volume in critically-ill patients is a challenging task as excessive or insufficient fluid administration increases patient morbidity and mortality\cite{smyrniotis2004,rivers2001}. Clinical research has shown that the variations in the anterior-posterior (AP) diameter of inferior vena cava (IVC) can be helpful in approximating a patient's volume status and whether or not they will benefit from additional intravenous fluids \cite{charron2006, durairaj2008, barbier2004}. Traditionally, the AP-diameter is manually estimated from portable ultrasound imagery - often a challenging and time-consuming task in the setting of poor image quality. Artifacts such as shadowing and speckle noise frequently result in  indistinct edges and gaps in the vessel walls reducing accurate estimation\cite{wang2014multiscale, sudha2009speckle}.\\
Speckle noise present in ultrasound imagery is traditionally considered to be a Rayleigh distributed multiplicative noise \cite{wagner1983statistics}. Hence, Rayleigh mixture models have been proposed as a potential solution for ultrasound image segmentation \cite{seabra2011rayleigh, pereyra2012segmentation}. However, it has been shown that due to the scattering population and signal processing, the speckle distribution deviates from Rayleigh\cite{tuthill1988deviations}. Furthermore, lossy compression algorithms present on many portable ultrasound machines further deviate the recorded clip from an idealized Rayleigh distribution.
\\
Active contours (ACs), as planar deformable models, are widely used for segmentation of ultrasound images \cite{kass1988, karami2016, liu2010probability, talebi2011medical, noble2010ultrasound}. ACs address image segmentation through minimization of energy functional(s) with their performance frequently dependent on a manually-defined initialization contour. In order to avoid local minima, the initiating contour needs to be as close as possible to the actual contour. ACs can be combined with other segmentation algorithms as a coarse-to-fine strategy to reduce the impact of the initial contour on segmentation error \cite{yim2003, ali2012}. Researchers have addressed the challenge of IVC segmentation using this strategy by using template matching method as the coarse segmentation and AC as the fine-tuning (TMAC) \cite{nakamura2013}. Unfortunately, this approach fails when the IVC undergoes large frame-to-frame variations commonly present on portable machines with lower frame rates (e.g. 30 frames-per-second). Additionally, ACs continue to perform poorly in the context of fuzzy or unclear boundaries as is commonly the case for the IVC.
\\
Given that the cross-section of the IVC is largely convex, the IVC contour can be represented in polar coordinates and consequently, polar active contours appear as a promising solution for IVC segmentation \cite{baust2012}. In  \cite{karami2017a}, a polar AC model based on the third centralized moment (M3) was proposed for segmentation of IVC images. Unfortunately,  M3 algorithms roughly estimates the cross-sectional area (CSA) of the IVC and fails with poor quality images.\\
Clinically the CSA of the IVC is an optimal approach to accurately assess a patient's volume status, but all existing approaches fail to accurately estimate the CSA. Hence, clinicians instead of the the whole CSA of the IVC, measure its AP-diameter. In this paper, we propose using an active circle algorithm incorporating a novel evolution functional to estimate the AP-diameter of the IVC across a spectrum of image qualities. In addition, the proposed evolution functional appears promising for other ultrasound image segmentation problems as well.
\\
The remainder of this paper is organized as follows - Section II discusses the background and related work. The proposed active circle algorithm is presented in Section III while experimental results are in Section IV and concluding remarks in Section V.
\section{Background and Related work}
\subsection{IVC Image}
Fig. \ref{fig:image} displays a typical ultrasound image of the IVC, with a circle fitted on the IVC AP-diameter shown in yellow color. For demonstrating the circle model which is later used in this work, we have also displayed two scaled circles, concentric with the yellow circle with scaling factors $S=0.75$ and $S=1.5$ shown in red and green colors, respectively. As in Fig. \ref{fig:image}, one can see that the IVC boundaries are generally fuzzy and are unclear. On the other hand, one can see that the inside of IVC is generally hypoechoic and the outside is hyperechoic indicating that the inside and outside of the IVC have distributions with different means. Fig. \ref{fig:histos} illustrates the probability density functions (PDFs) of the intensity levels inside the three circles shown in Fig. \ref{fig:image}, where the x-axis is the normalized pixel intensity which is between 0 and 1. From Fig. \ref{fig:histos}-(a) and (b), one can see that the intensity distribution inside the IVC contour is rather sparse than continuous. Assume that the intensity levels inside and outside the IVC have PDFs $F_{in}$ with mean $m_{in}$ and $F_{out}$ with $m_{out}$, respectively. It is obvious that $m_{in} < m_{out}$. This indicates that regardless of the PDFs $F_{in}$ and $F_{out}$, their distinct means can be used for image segmentation.
\begin{figure}[t!]
\centering
\includegraphics[width=0.5\linewidth]{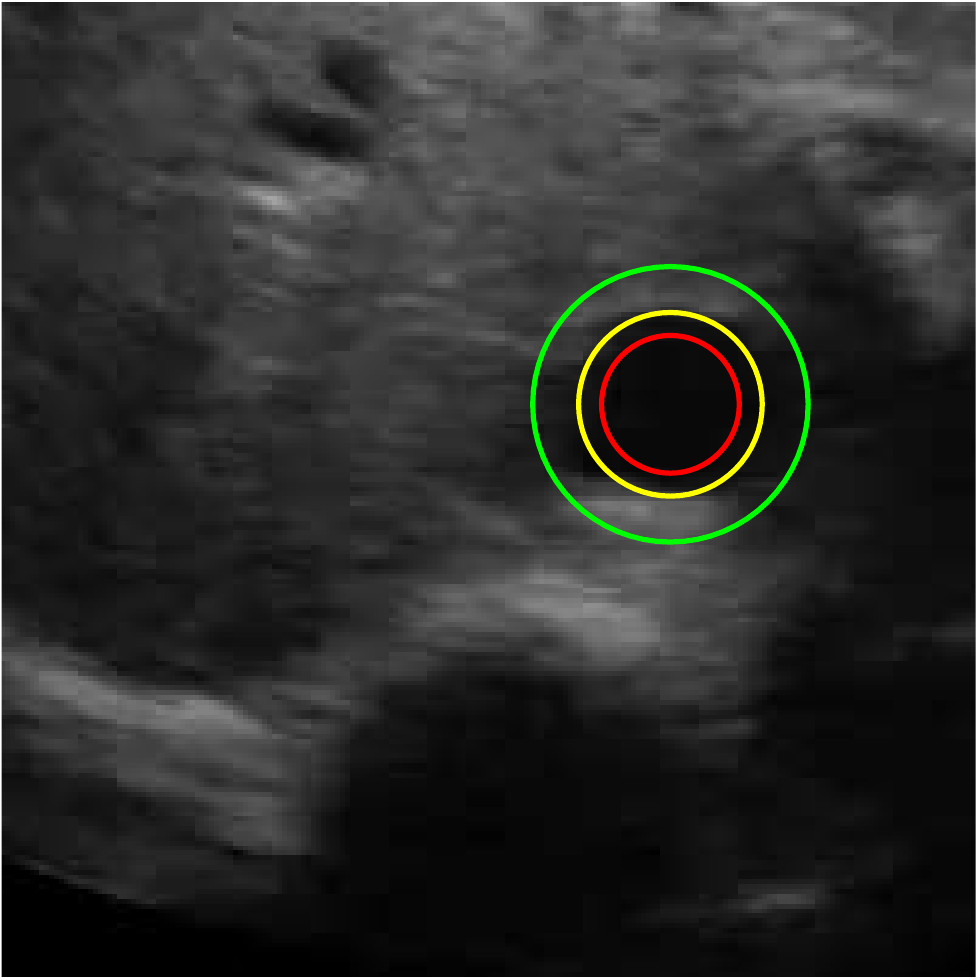}
\vspace{0.2in}
\caption{A typical ultrasound image of IVC with circle evolution model.}\label{fig:image}
\end{figure}
\begin{figure}[t!]
\centering
\includegraphics[width=1\linewidth]{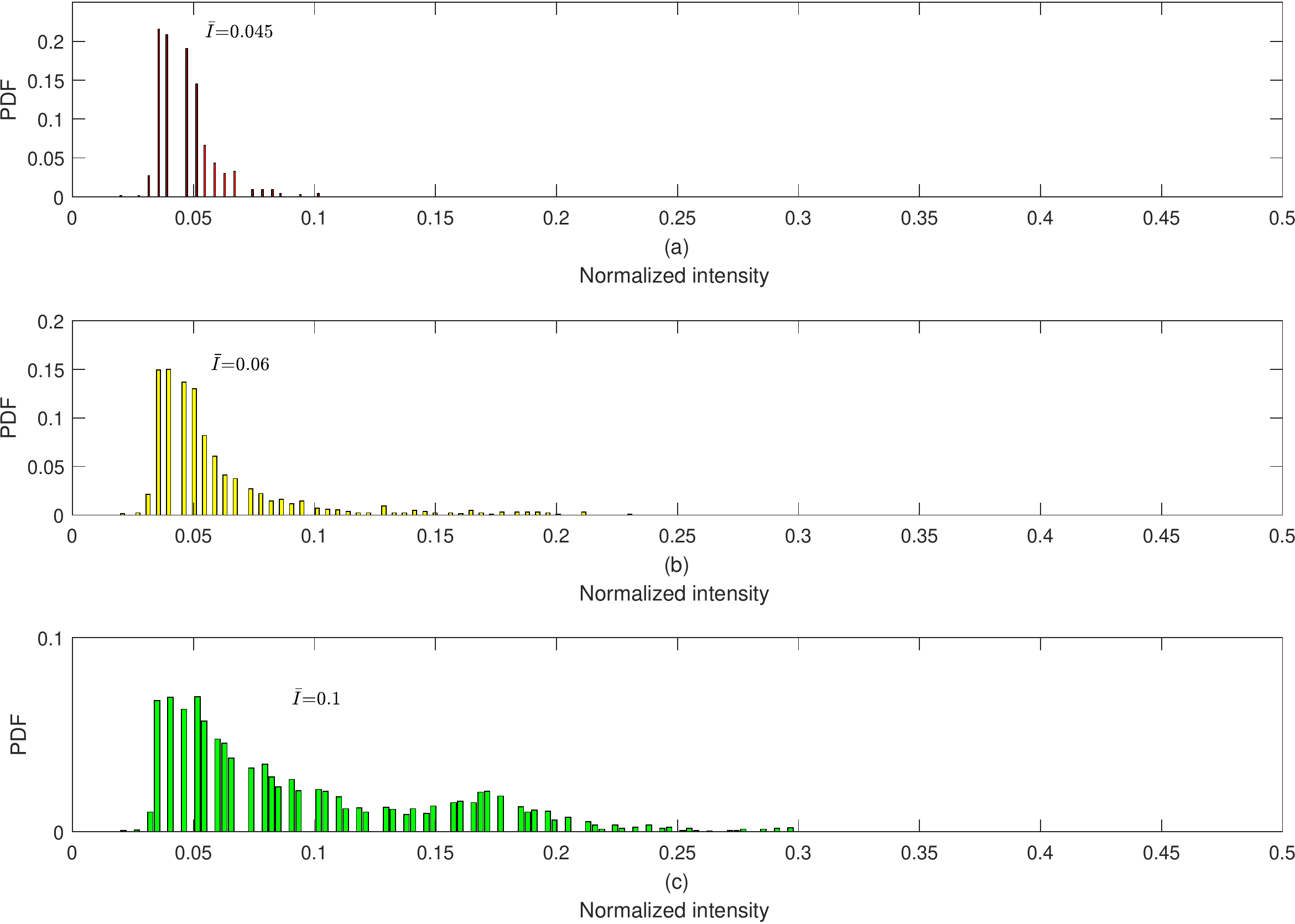}
\caption{PDF of the intensity levels inside the disks concentric with the IVC (see Fig. \ref{fig:image}) with radius equal to (a): 75\%, (b): 100\%, and (c): 150\% of the IVC AP-diameter.}
\label{fig:histos}

\end{figure}

\subsection{Energy and Evolution Functionals}
In traditional level set methods for a given energy functional $C$, the evolution functional is obtained as \cite{li2005level,cremers2006kernel}:
\begin{equation}
    \frac{\partial C}{\partial t} =-F(|\nabla C|),\label{eq:level}
\end{equation}
\noindent where the function $F$ depends on the image and is usually defined as a linear function. In this paper, we use the following linear function \cite{li2005level,cremers2006kernel}
\begin{equation}
    E =-|\nabla C| \vec{n},\label{eq:level2}
\end{equation}
\noindent where $\vec{n}$ is the vector normal to the contour.
\\
Two conventional functionals, widely employed in variational ACs, are based on the mean and the variance of the intensities \cite{yezzi2003variational}.
\\
\noindent \textit{Functional Based on Mean:} Assuming $u$ and $v$ represent the mean intensity levels inside and outside the IVC, respectively, the energy functional is defined as \cite{yezzi2003variational} 
\begin{equation}
    C_{mean}=\frac{\alpha}{2}(u-v)^2,\label{eq:mean}
\end{equation}
\noindent where $\alpha$ is a positive weighting factor. Using (\ref{eq:level}), the evolution functional corresponding to (\ref{eq:mean}) is obtained as \cite{yezzi2003variational}
\begin{equation}
E_{mean}=\alpha(u-v)(\frac{I-u}{A_u}+\frac{I-v}{A_v})\vec{n},\label{evol:mean}
\end{equation}
where $I$ is the intensity at the contour point and $A_u$ and $A_v$ are the areas inside and outside the contour, respectively.
\\
\noindent \textit{Functional Based on Variance:} Assuming $\sigma_u^2$ and $\sigma_v^2$ as the variances of intensity levels inside and outside the IVC respectively, the energy functional is defined as \cite{yezzi2003variational}
\begin{equation}
    C_{var}=\sigma_u^2+\sigma_v^2,\label{eq:var}
\end{equation}
Using (\ref{eq:level}), the evolution functional corresponding to (\ref{eq:var}) is obtained as
\begin{equation}
E_{var}=\alpha(\frac{I^2-u^2-\sigma_u^2}{A_u}-\frac{I^2-v^2-\sigma_v^2}{A_v})\vec{n},\label{evol:var}
\end{equation}
\\
where $u$ and $\sigma_u^2$ are the mean and variance of the intensities for the pixels inside the contour while $v$ and $\sigma_v^2$ represent for the ones outside the contour.
\vspace{-0.2cm}
\section{Proposed Algorithm}
\subsection{Why a Circular Model?}
IVC images can be segmented with polar ACs as in \cite{karami2017a}. When a polar AC with $N$ contour points is used, the number of parameters that has to be estimated is $N$, i.e., one radial distance for each contour point, while with traditional Cartesian ACs, $2N$ parameters have to be estimated, i.e., two for $x$ and $y$ coordinates of each contour points. This makes polar contour models less complex and more accurate for vessel segmentation. On the other hand, to estimate the AP-diameter of the IVC, it is not necessary to fully segment the IVC contour; and hence, we can exploit a reduced model such as a circle which only has three parameters and, consequently, can be estimated more precisely. Fig. \ref{fig:test} shows four sample ultrasound images from IVCs with different shapes. For each case, the IVC contour is highlighted with yellow colors and its corresponding AP-diameter is shown with green colors. From Fig. \ref{fig:test}, one can see that regardless of the shape of the IVC, a circle fitted inside it can accurately model and estimate the IVC AP-diameter. 
\begin{figure}
\centering
\includegraphics[width=0.8\linewidth]{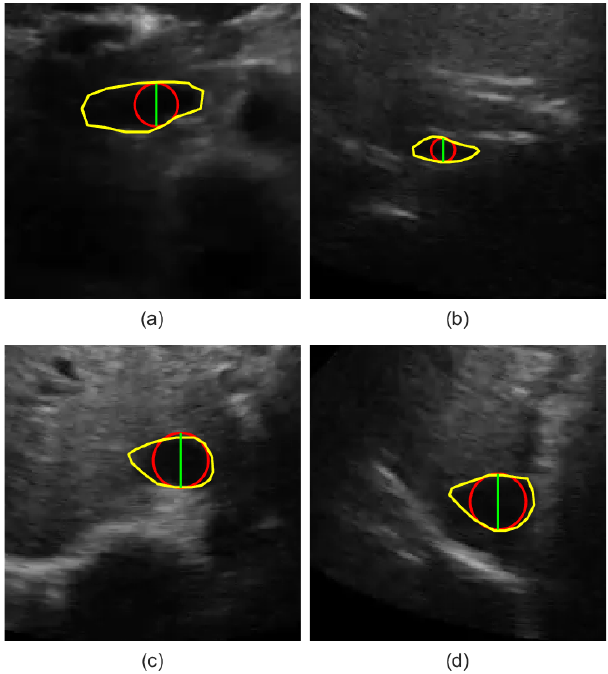}
\vspace{5mm}
\caption{Estimation of AP-diameter using circle fitting in IVC images with different shapes and qualities.}
\label{fig:test}
\end{figure}
\subsection{Proposed Evolution Functional}
After exploring a variety of energy and evolution functionals, we heuristically found the following evolution functional to be useful for segmentation of ultrasound images.
\begin{equation}
E=\alpha(u-v)(2I-u-v).\label{eq:evol}
\end{equation}
With the model defined in Section II.A, it is easy to see that if the contour is entirely inside the IVC, then $\bar{I}=u=m_{in}$, with $\bar{I}$ being the average intensity for the points on the contour. Consequently, the evolution functional for contour points is a random variable with the mean value
\begin{equation}
    \bar{E}=(u-v)(2\bar{I} -u-v)=(u-v)^2,
\end{equation}
\noindent resulting in the contour consistently expanding toward the actual IVC boundary. On the contrary, if the contour is entirely outside the IVC, then $\bar{I}=v=m_{out}$ and the evolution functional for contour points is a random variable with the mean value
\begin{equation}
    \bar{E}=(u-v)(2\bar{I} -u-v)=-(u-v)^2,
\end{equation}
\noindent resulting in the contour consistently shrinking toward the actual IVC boundary.\\
To support this finding, we computed the sensitivity of the proposed evolution functional to translation and scaling of the fitted circle. \\
Fig. \ref{fig:evol} shows this results with the y-axis being the average contour evolution for the points on fitted circle with $\alpha=10^{-3}$. Fig. \ref{fig:evol}-(a) shows this averaged functional when the circle diameter equals to AP-diameter but its center is shifted by $\delta_x$ along the x-axis. From \ref{fig:evol}-(a), one can see that when the fitted circle is shifted toward the left, i.e, with $\delta x < 0$, the average evolutional functional is positive, moving the fitted circle to the right, and vice versa. The circle evolution only stops when the average evolutional functional is zero which occurs at $\delta x=0$.  One can see a similar result for the transverse axis in Fig. \ref{fig:evol}-(b). This proves that with the proposed evolutional functional, the algorithm reaches equilibrium when the circle centers on the IVC. Fig. \ref{fig:evol}-(c) shows this result versus the circle diameter, i.e., functional averaged over the points on the circle with diameter $D$ and concentric with the IVC center. From Fig. \ref{fig:evol}-(c), one can see that when the circle diameter is less than $D_{AP}$ of the IVC, the evolutional functional is positive demonstrating that the contour is expanding toward the actual IVC boundaries. This proves that with the proposed functional, the algorithm reaches its equilibrium if $D=D_{AP}$, i.e., when the diameter of the circle equals to the actual AP-diameter of the vessel.
\begin{figure}[t!]
\centering
\includegraphics[width=1\linewidth]{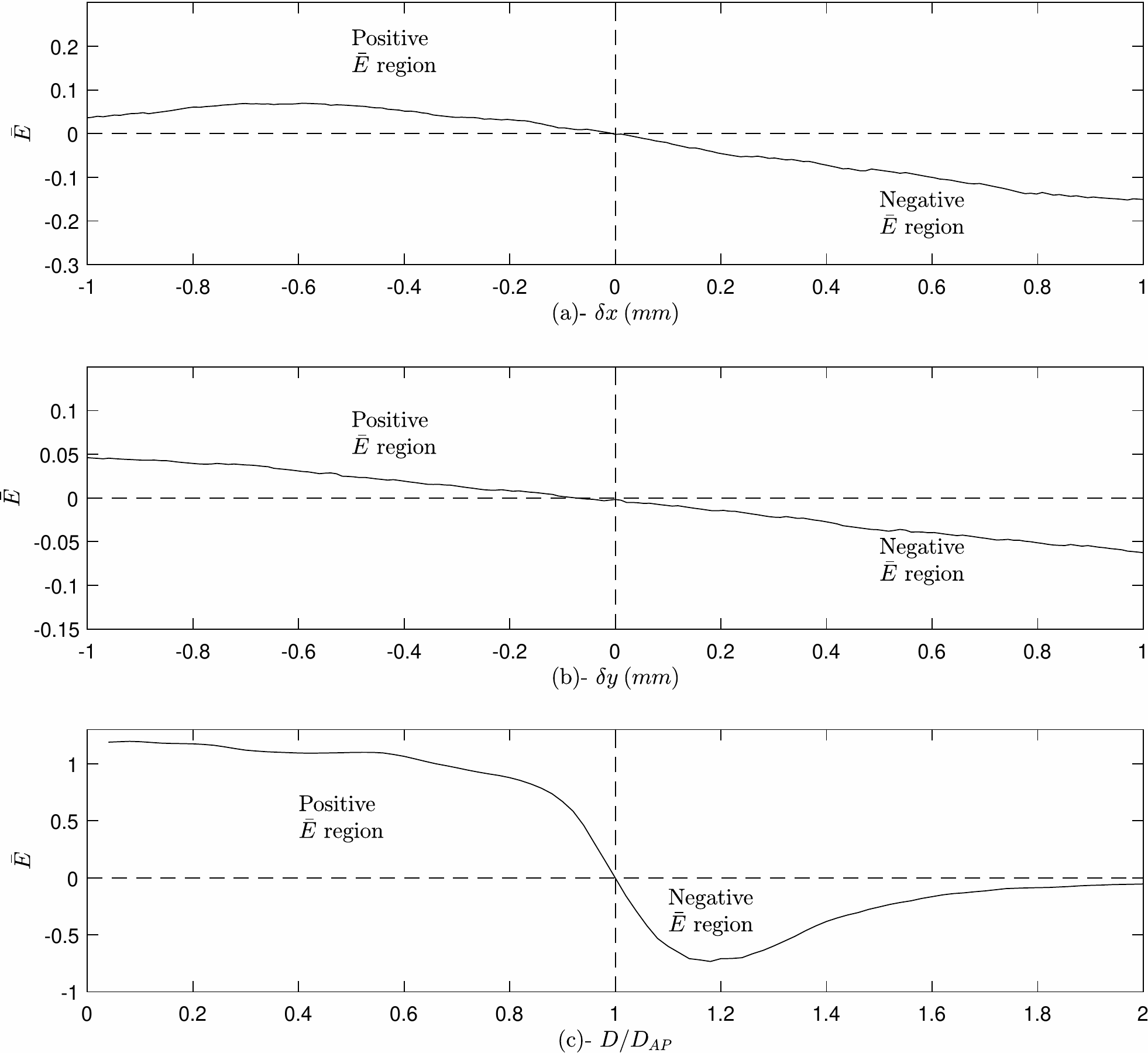}
\vspace{0.25in}
\caption{Proposed evolution functional versus (a)- $\delta x$, (b)- $\delta y$, and (c)- normalized circle diameter $D/D_{AP}$ with $\alpha=10^{-3}$. }
\label{fig:evol}
\end{figure}
\subsection{Circle Evolution}
The evolution functional is utilized to update the parameters of the circle with $R$ and $(x_c,y_c)$ as its radius and center, respectively.  Initially, the circle is sampled at $K$ points with polar angles $\theta_k=\frac{2k\pi}{N}$, $k=0,1,...,K-1$, where the normal vector and Cartesian coordinates corresponding to the $k$th sampled point notated as 
\begin{equation}
\vec{n}_k=[\cos(\theta_k), \sin(\theta_k)]^T, 
\end{equation}
\noindent and 
\begin{equation}
[x_k, y_k]^T=[x_c, y_c]^T+R\vec{n}_k,
\end{equation}
\noindent respectively. The evolution functional generates forces $f_k=\alpha(u-v)(2I_k-u-v)$ along the normal vectors $n_k, k=0,1,...,K-1$, as shown in Fig. \ref{fig:forces}.
\begin{figure}[t!]
\centering
\includegraphics[width=0.9\linewidth]{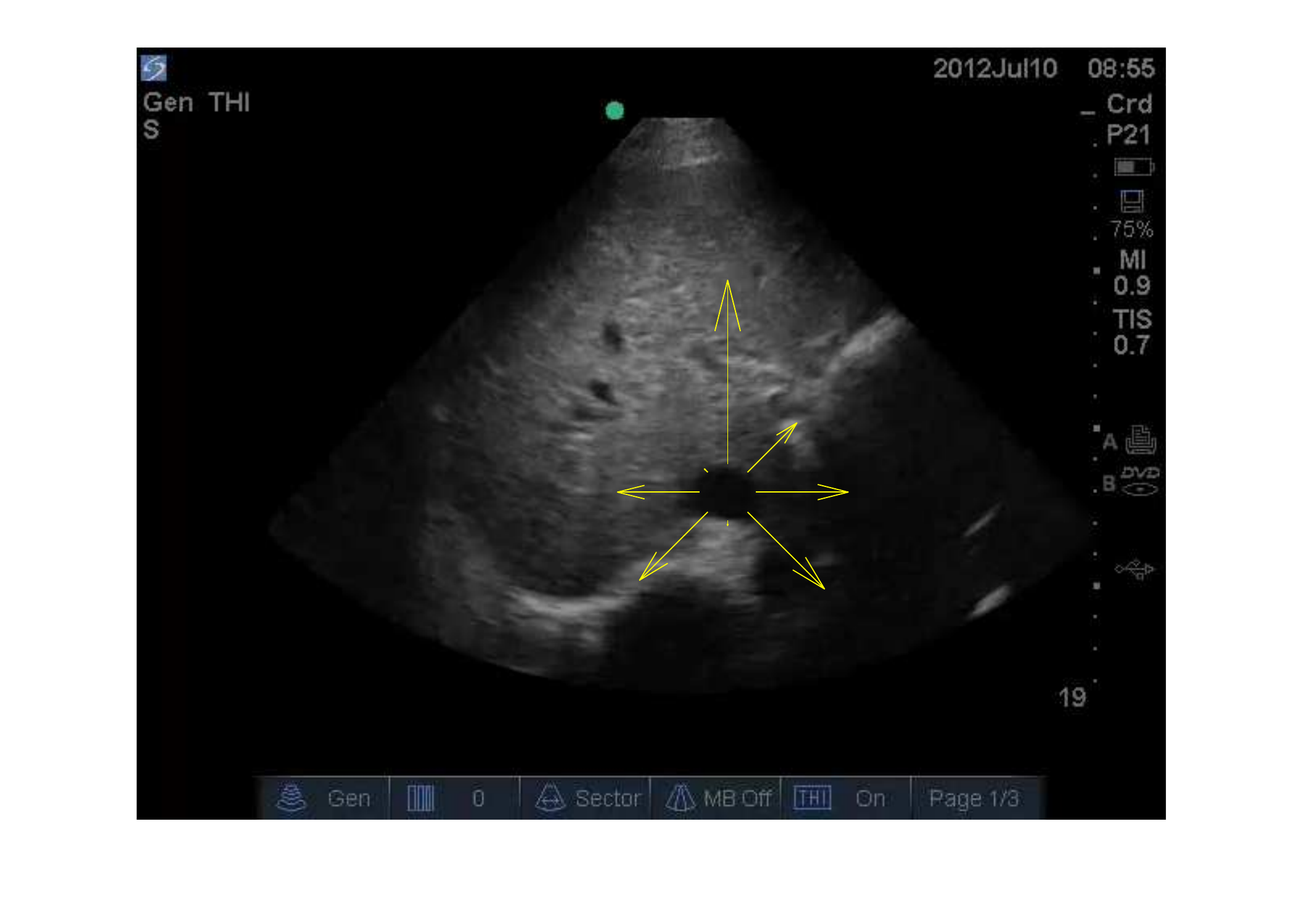}
\caption{The forces generated by the evolution functional in (\ref{eq:evol}) with $\alpha=0.05$ and $K=8$ (values are adjusted for illustration purposes).}
\label{fig:forces}
\vspace{-0.3cm}
\end{figure}
The forces shift the sampled contour points to new positions governed by
\begin{equation}
[\tilde{x}_k, \tilde{x}_k]=[x_c, y_c]^T+(R+f_k)\vec{n}_k.\label{shifted}
\end{equation}
\noindent where $f_k$ is the value of the evolution functional at $k$th contour point. Obviously, the shifted contour points are not on a circle anymore; hence, a new circle needs to be fitted to the updated contour points. This is accomplished through minimization of the following energy functional:
\begin{equation}
C_{circle}=\sum\limits_{k=0}^{K-1}[(\tilde{x}_k-\tilde{x}_c-\tilde{R}\cos(\theta_k))^2+(\tilde{y}_k-\tilde{y}_c-\tilde{R}\sin(\theta_k))^2],\label{eq:newdisk}
\end{equation}
where $[\tilde{x}_c, \tilde{y}_c]$ and $\tilde{R}_c$ are the center and radius of the updated circle. By minimizing $C_{circle}$, the new values for the center and radius of the circle are obtained using the following theorems.
\begin{theorem}
For a given set of forces, the circle center is shifted by the average of the force vectors $f_k\vec{n}_k$.
\end{theorem}
\begin{proof}\let\qed\relax
To find $\tilde{x}_c$ and $\tilde{y}_c$, the energy functional ($C_{circle}$) in (\ref{eq:newdisk}) is minimized by setting its gradient to zero as:
\begin{equation}
\frac{\partial C_{circle}}{\partial \tilde{x}_c}=-2\sum\limits_{k=0}^{K-1}(\tilde{x}_k-\tilde{x}_c-\tilde{R}\cos(\theta_k))=0,\label{eq:forcesx}
\end{equation}
\begin{equation}
\frac{\partial C_{circle}}{\partial \tilde{y}_c}=-2\sum\limits_{k=0}^{K-1}(\tilde{y}_k-\tilde{y}_c-\tilde{R}\sin(\theta_k))=0.\label{eq:forcesy}
\end{equation}
Since $\sum\limits_{k=0}^{K-1}\cos(\theta_k)=\sum\limits_{k=0}^{K-1}\sin(\theta_k)=0$, by subsisting (\ref{eq:newdisk}) in (\ref{eq:forcesx}) and (\ref{eq:forcesy}), one can easily find
\begin{equation}
[\tilde{x}_c, \tilde{y}_c]=[x_c, y_c]+\frac{1}{K}\sum\limits_{k=0}^{K-1}f_k\vec{n}_k,\label{eq:newcenter}
\end{equation}
\noindent highlighting that the circle center is shifted by the average of the force vectors $f_k\vec{n}_k$.
\end{proof}
\begin{theorem}
For a given set of forces, the circle radius is modified with the average of the force values $f_k$.
\end{theorem}
\begin{proof}\let\qed\relax
To find the new circle radius $\tilde{R}$, the energy functional ($C_{circle}$) in (\ref{eq:newdisk}) is minimized by setting its gradient to zero as:
\begin{equation}
\begin{split}
  &\frac{\partial C_{circle}}{\partial \tilde{R}}=-2\sum\limits_{k=0}^{K-1}[(\tilde{x}_k-\tilde{x}_c-\tilde{R}\cos(\theta_k))\cos(\theta_k)\\&+(\tilde{y}_k-\tilde{y}_c-\tilde{R}\sin(\theta_k))\sin(\theta_k)]=0.\label{eq:forceR}  
\end{split}
\end{equation}
By substituting (\ref{eq:newcenter}) in (\ref{eq:forceR}), one determines that
\begin{equation}
\tilde{R}=R+\frac{1}{K}\sum\limits_{k=0}^{K-1}f_k.\label{eq:newradius}
\end{equation}
\end{proof}
\subsection{Proposed Active Circle Algorithm}
In this paper we set $K=32$ and $\alpha=10^{-4}$. Note that with a larger value of $K$, the estimation accuracy is improved at the cost of increased computational complexity. Similarly, with an smaller value of $\alpha$ the accuracy is improved at the cost of increased number of iterations required to reach the convergence. The flowchart of the algorithm is shown in Fig. \ref{fig:flowchart}. As the first step, the proposed algorithm requests to manually locate the IVC. This is simply performed by a mouse click on a point inside the IVC. The initial circle centers at this selected point and to avoid convergence to a wrong boundary, its radius is assumed to be as small as 6 pixels. In the second step, we compute the forces $f_k$ using (\ref{eq:evol}). In the third step, these forces to evolve the circle parameters using (\ref{eq:newcenter}) and (\ref{eq:newradius}). Steps two and three are repeated until either a convergence or maximum number of iterations, i.e., 5000 iterations is achieved. In this paper, we assume the algorithm has converged, if the largest force computed in the second step is less than $10^{-3}$ pixels. This process is repeated for the next frames of the videos. 

\begin{figure}
    \centering
    \includegraphics{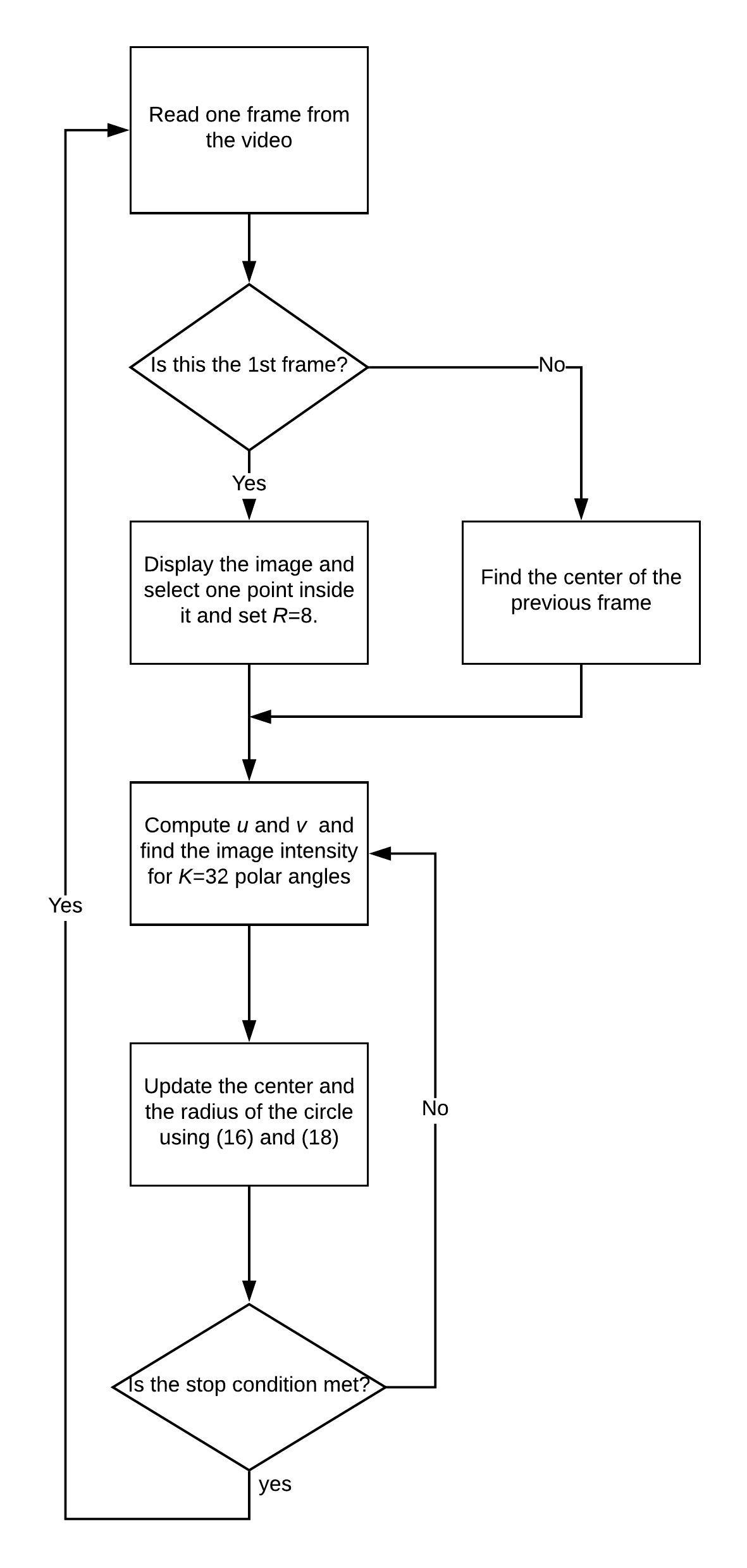}
    \caption{Flowchart of the proposed active-circle algorithm.}
    \label{fig:flowchart}
\end{figure}

\subsection{Complexity Analysis}
The complexity of the proposed active circle algorithm is obviously much less than the state-of-the-art algorithms. The algorithm obtains its low complexity from: 1) the simplicity of the proposed evolutional 2) the simplicity of the circle model which only has three parameters to estimate. In this Section, we estimate the computational complexity of the active-circle algorithm using the number of floating point operations (flops) \cite{watkins2004fundamentals}. Assume $N_{iter}$ as the average number of iterations and $A_{max}=5000$ pixels as the maximum area of the circle. From equations (\ref{eq:evol}),  (\ref{eq:newcenter}), and (\ref{eq:newradius}), one can see that the maximum number of flops for each frame is $N_{flops}=(2A_{max}+3K+5)N_{iter}\approx 50$ million flops, i.e., with an average i7-3770 Intel processor, estimation of AP-diameter from each frame requires less than 1 millisecond.

\section{Results}
The experimental data was collected from eight healthy male subjects with ages from 21 to 35. The study protocol was reviewed and approved by the Health Research Ethics Authority. The IVC was imaged in the transverse plane using a portable ultrasound (M-Turbo, Sonosite-FujiFilm) with a phased-array probe (1-5 Mhz). Each video has a frame rate of 30 fps, scan depth of 19cm, and a duration of 15 seconds (450 frames/clip). Fig. \ref{fig:3images} depicts the first frame of three subjects with different qualities which are graded, based on their quality and clinical impression, by three experts, including Dr. Andrew Smith as a point-of-care ultrasound expert and two additional clinicians. Note that although the sample image in Fig. \ref{fig:3images}-(b) seems to have a better quality than the one in Fig. \ref{fig:3images}-(a), it is rated as average quality due to its more fuzzy boundaries which degraded the accuracy of the algorithm.

\begin{figure}[t!]
\centering
\includegraphics[width=0.97\linewidth]{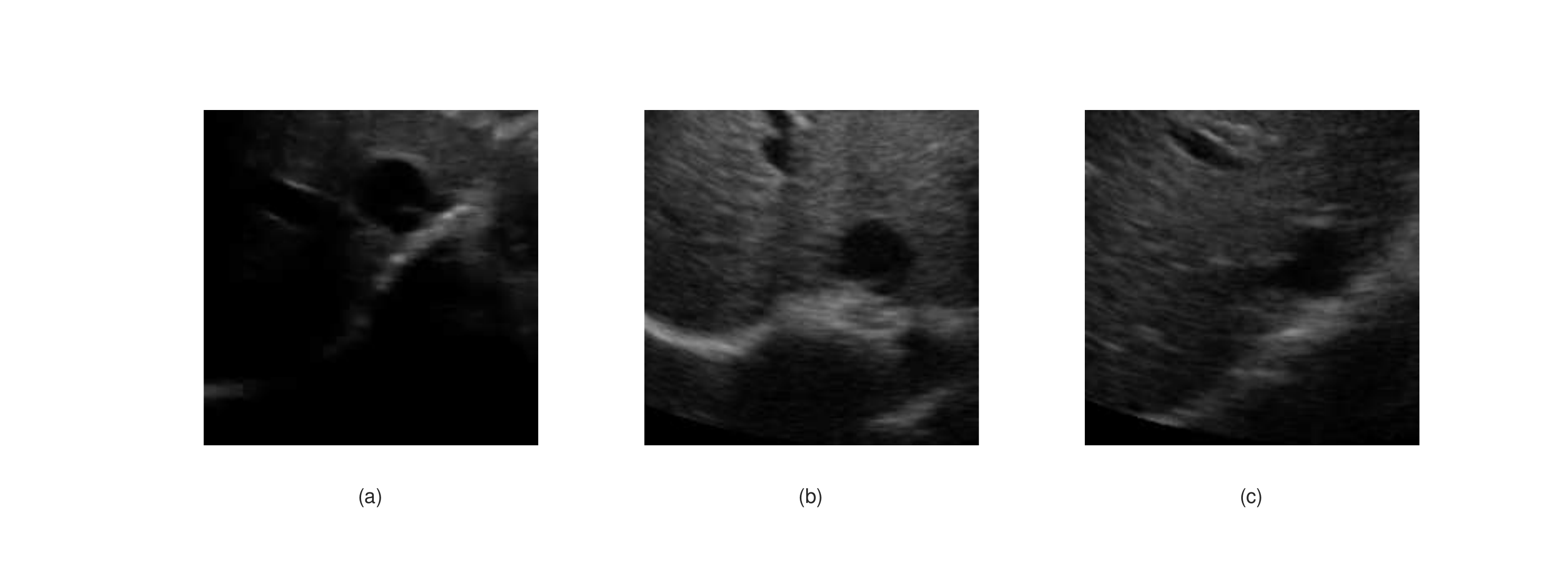}
\vspace{-3mm}
\caption{First frame of three sample IVC videos rated as (a)- good, (b)- average, and (c)- poor quality videos.}
\label{fig:3images}
\vspace{-0.3cm}
\end{figure}

\subsection{Comparison of the proposed functional with state-of-the-art functionals}
In this Section, we first compare the proposed evolution functional with manual measurements made by Dr. Andrew Smith, and two state-of-the-art evolution functionals in eqs. (\ref{evol:mean}) and (\ref{evol:var}). From Fig. \ref{fig:func12}, one can see that in all three investigated videos, the active circle algorithm using the two state-of-the-art functionals fail to track the relatively fast AP-diameter variations in the first video. This is because, in IVC images $A_u \ll A_v$, and hence in eqs. (4) and (6), the second term is dominated by the first term. Consequently, the balance between the intensities inside and outside the contour is not established.
Note that in Fig. \ref{fig:func12}-(c) which presents the results for the third subject, due to the extremely poor quality of this video, the manual measurement is partially missing for the frames in range 297 to 332, as the expert was unable to measure the AP-diameter. The proposed algorithm still estimate the AP-diameter although there is no ground truth for these frames. 
\begin{figure}[t!]
\centering
\includegraphics[width=1\linewidth]{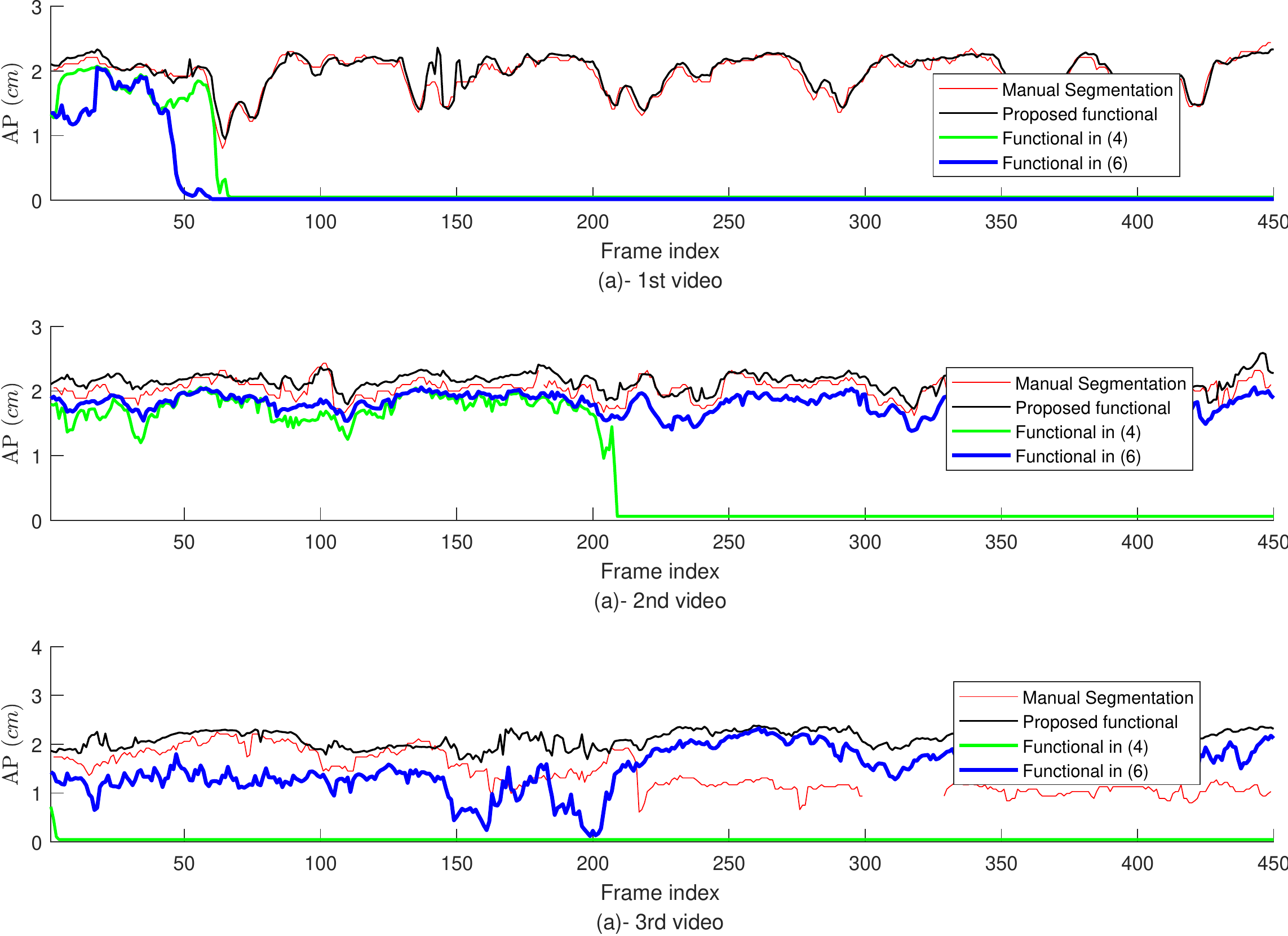}
\caption{IVC diameter of a typical IVC video as measured by the proposed functional and two functionals in equations (\ref{evol:mean}) and (\ref{evol:var}).}
\label{fig:func12}
\end{figure}
\subsection{Influence of the parameter $\alpha$ on the accuracy of the proposed algorithm}
In this section, we investigate the sensitivity of the proposed algorithm to the value of $\alpha$ for the three videos depicted in Fig. \ref{fig:3images}. For this study, we use root mean square (RMS) of error as the performance criterion, with the error $e$ defined as the the difference between the AP-diameter estimated from the proposed algorithm and the manual measurement. For the first two subjects, the RMS of error is calculated over all 450 frames, but for the third video, it is calculated over the first 150 frames, where the manual measurement seems to be reliable.  From Fig. \ref{fig:sens}, one can easily see that in all three cases, the proposed algorithm performs well with $\alpha < 5 \times 10^{-4}$. Hence, for the rest of our experiments, we set $\alpha=10^{-4}$. Note that with a smaller $\alpha$, more accurate results can be obtained in the cost of more number of iterations required for convergence. 
\begin{figure}[t!]
\centering
\includegraphics[width=0.6\linewidth]{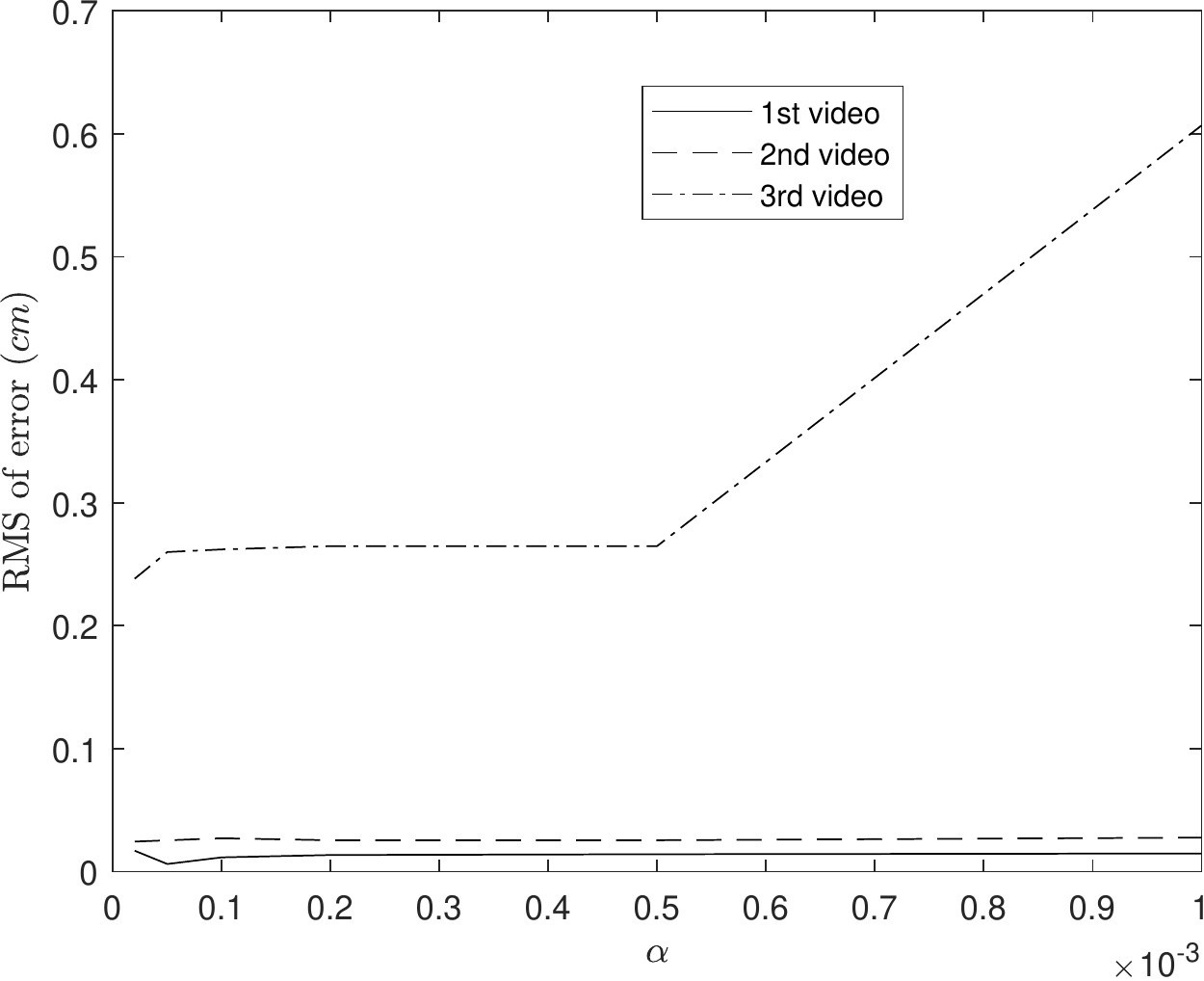}
\vspace{-3mm}
\caption{The RMS of error for the proposed algorithm w.r.t. the parameter $\alpha$ for the three sample videos depicted in Fig. \ref{fig:3images}.}
\label{fig:sens}
\vspace{-0.3cm}
\end{figure}

\subsection{comparison of the proposed algorithm with state-of-the-art segmentation algorithms}
The proposed active circle algorithm was also compared with expert manual measurement, the two classic AC algorithms - Chan-Vese \cite{chan2001} and Geodesic \cite{caselles1997}, and four state-of-the-art polar ACs- PSnake \cite{de2014} and variational polar AC \cite{baust2011}, M3 algorithm \cite{karami2017a}, and Ad-PAC \cite{karami2018}, and star-Kalman algorithm \cite{guerrero2007real}.
\\
Fig. \ref{fig:res}-(a) presents the results for the video depicted in Fig. \ref{fig:3images}-(a). This figure illustrates that only the active circle accurately tracks and segments the IVC and follows the variations in manual extraction. M3 algorithm fails to track, whereas the others struggle to capture the temporal variation.

\begin{figure}[t!]
\centering
\includegraphics[width=1\linewidth]{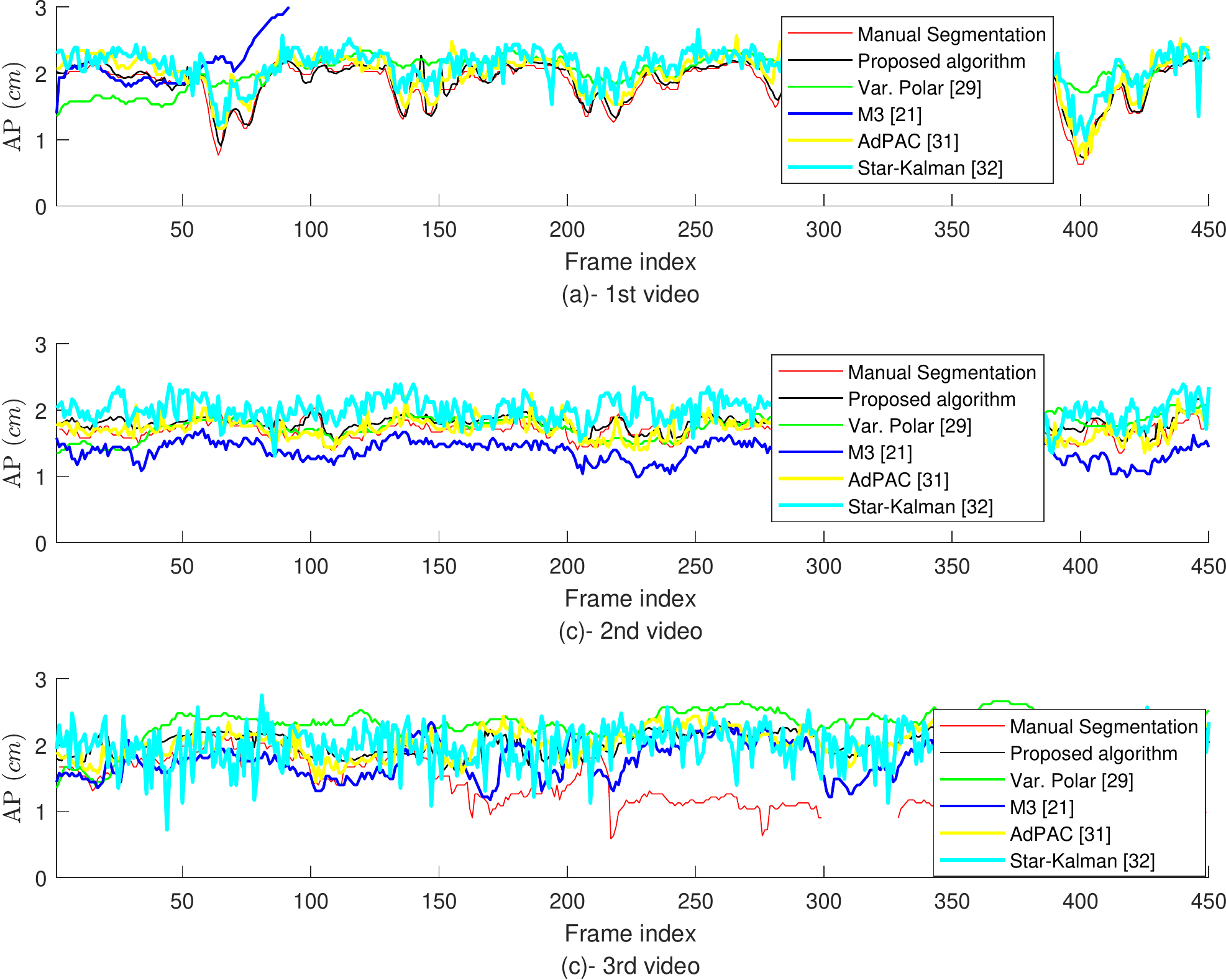}
\caption{AP-diameter for the three samples videos depicted in Fig. \ref{fig:3images}, as measured by the proposed algorithm, manual extraction, and four other algorithms.}
\label{fig:res}
\end{figure}
Fig. \ref{fig:res}-(b) details the results for the video depicted in Fig. \ref{fig:3images}-(b) which suffers from artifact partially occluding the vessel wall. Here, it is evident that the proposed algorithm again tracks the temporal variations of the manual extraction.

Fig. \ref{fig:res}-(c) details the results for the video depicted in Fig. \ref{fig:3images}-(c) which has a poor quality (the worst among all eight investigated videos). In this case, both the proposed algorithm and manual extraction roughly estimated the AP-diameter for the first 150 frames and then they failed.\\
Table 1 presents the RMS of the AP-diameter estimation error for the proposed algorithm and other seven algorithms and for all eight studied subjects, with the last column showing the results averaged over all subjects. Note that videos no. 1 to 3 are the ones depicted in Fig. \ref{fig:3images}. Except the third subject where the manual measurement is reliable only for the first 150 frames (see Fig. \ref{fig:res}-(c)), for the other seven subject, the RMS of error is calculated over all 450 frames. From this table, one can see that for all eight studies subjects, the proposed algorithm outperforms other seven methods. The second best performance is obtained from M3 algorithm. One can also see that the Star-Kalman algorithm which fits an ellipse to the IVC contour performs much poorer than the proposed active-circle algorithm. The worst average performance is obtained from the two classic AC algorithms, i.e., Chan-Vese and Geodesic AC, showing that polar ACs perform in average more accurate than the Cartesian ACs because their larger degree of freedom, i.e, the number of parameters that has to be estimated.

\begin{table}[H]
\caption{RMS of the AP-diameter estimation error obtained using the proposed algorithm and seven other methods.}
\begin{center}
 \begin{tabular}{||c|c|c|c|c|c|c|c|c|c||} 
 \hline
 \backslashbox{Method}{Subject no.} & 1 & 2 & 3 & 4  & 5 & 6 & 7 & 8 & Ave\\ [0.5ex] 
 \hline\hline
Proposed algorithm & 0.10 & 0.16 & 0.26 & 0.16 & 0.25 & 0.25 & 0.1 &  0.11 & 0.17\\ 
   \hline
  Var. Polar [29] & 0.38 & 0.25 & 0.49 & 0.27 & 0.42 & 0.28 & 0.24 & 0.52 & 0.36\\
  \hline
 Psnake [28] & 0.66 & 0.50 & 0.52 & 0.5 & 0.4 & 0.33 & 0.49 & 0.71 & 0.51\\
  \hline
 M3 [21] & 0.53 & 0.30 & 0.24 & 0.30 & 0.31 & 0.32 & 0.41 & 0.42 & 0.35 \\
 \hline
 AdPAC [31] & 0.21 & 0.17 & 0.2 & 0.17 & 0.16 & 1.1 & 1.8 & 0.13 & 0.39\\
 \hline
 Star-Kalman [32] & 0.34 & 0.45 & 0.44 & 0.45 & 0.58 & 0.54 & 0.37 & 0.25 & 0.43\\
 \hline
 Chan-Vese [26] & 1.77 & 0.39 & 0.54 & 0.39 & 0.38 & 0.33 & 0.55 & 0.73 & 0.66\\
  \hline
 Geodesic [27] & 1.89 & 0.17 & 0.54 & 0.17 & 0.38 & 0.33 & 0.55 & 0.73 & 0.60\\
 \hline
\end{tabular}
\end{center}
\end{table}
\vspace{-0.2cm}
\subsection{More Comparison Metrics}
In table 2, we compare the proposed algorithm and manual estimation using more metrics, where some of these metrics are also clinically useful to assess patient's status. In this table, $e_{ave}$ is the average estimation error, averaged over the frames in each video; $\sigma_{e}$ is the standard deviation of the estimation error $e$ for the frames in each video, and $|e|_{max}$ is the maximum error. Note that $e_{ave}$, except for the third subject where only the manual measurements for the first 150 frames are reliable, is calculated over all 450 frames. Since the values of $D_{max}$ and $D_{min}$ are useful for the medical purposes, we have also compared them obtained from both the algorithm and manual measurement. Note that subjects no. 1 to 3 are the ones depicted in Fig. \ref{fig:3images}. Furthermore, note that for subject no. 3, i.e., the lowest quality video, these metrics are only calculated for the first 150 frames, where the manual extraction is relatively reliable, while for other seven subjects, they are computed over all 450 frames. From table 2, one can see that except for the fourth subject, in other cases, $D_{max}$ estimated from the proposed algorithm is very close to the one obtained from the manual measurement and for 5 out of the eight studied subjects, the error for estimation of $D_{max}$ is less than 2mm. Furthermore, in all cases the average error is positive indicating that the proposed algorithm usually overestimates the AP-diameter although this bias is very small and is mostly less than 2mm. Note that this overestimation is mainly due to the fact that the expert extracts the AP-diameter from the inside boundary of IVC. The amount of bias can simply be controlled by adding a constant value to the functional. In all cases, the value $\sigma_{e}$ is small indicating that the proposed algorithm performs consistent over different frames of each video.

\begin{table}[H]
\caption{Summary of the results of eight studied IVC videos.}
\begin{center}
 \begin{tabular}{||c|c|c|c|c|c|c|c|c||} 
 \hline
 \backslashbox{Metric}{Subject no.} & 1 & 2 & 3 & 4 & 5 & 6 & 7 & 8\\ [0.5ex] 
 \hline\hline
 $e_{ave}$ (cm) & 0.03 & 0.07 & 0.07 & 0.18 & 0.29 & 0.37 & 0.04 & 0.08\\ 
   \hline
$\sigma_{e}$ (cm) & 0.06 & 0.07 & 0.14 & 0.08 & 0.08 & 0.13 & 0.09 & 0.11 \\
  \hline
 $|e|_{max}$ (cm) & 0.06 & 0.29 & 0.43 & 0.48 & 0.57 & 0.75 & 0.41 & 0.44 \\
  \hline
 $D_{max}$ alg. (cm) & 2.39 & 2.59 & 2.3 & 1.93 & 2.35 & 2.29 & 2.14 & 2.33 \\
 \hline
 $D_{min}$ alg. (cm) & 0.67 & 0.75 & 1.7 & 1.37 & 1.69 & 1.38 & 1.54 & 1.16 \\
 \hline
 $D_{max}$ man. (cm) & 2.43 & 2.43 & 2.25 & 2.41 & 2.36 & 2.08 & 2.21 & 2.39 \\
 \hline
 $D_{min}$ man. (cm) & 0.7 & 0.66 & 1.36 & 0.94 & 1.34 & 1.22 & 1.42 & 1.21\\

 \hline
\end{tabular}
\end{center}
\end{table}
\vspace{-0.2cm}

\subsection{Influence of speckle removal filtering on the performance of the proposed algorithm}
In this subsection, we investigate the influence of different speckle-removal filters on the performance of the proposed algorithm. For this study, we used the proposed algorithm for the images smoothed by either of the the following filters: Bilateral filter \cite{tomasi1998}, 3D block matching filter (BM3D) \cite{dabov2007}, speckle reducing anisotropic diffusion (SRAD) \cite{Yu2002}, Wiener filter \cite{loizou2008}, and median filter.\\
Fig. \ref{fig:filter} shows the AP-diameter estimated using the proposed algorithm with each of the  above speckle removal filters and have shown their results for the subjects depicted in Fig. \ref{fig:3images}. From Fig. \ref{fig:filter}, one can see that none of these filtering techniques significantly improve the performance of the proposed algorithm. This is mainly due to the fact that the speckle noise in ultrasound images includes useful information that can even improve the performance of computerized algorithms. On the other hand, as it was discussed in Section III, the proposed algorithm relies on the mean value of the local distributions which usually is not significantly changed by speckle removal filters. \\
Table 3 presents numerical results to validate this argument. One can see that except bilateral filter, other speckle-removal filters even degrade the performance and the improvement obtained from bilateral filter is negligible. 
\begin{figure}
    \centering
    \includegraphics[width=1\linewidth]{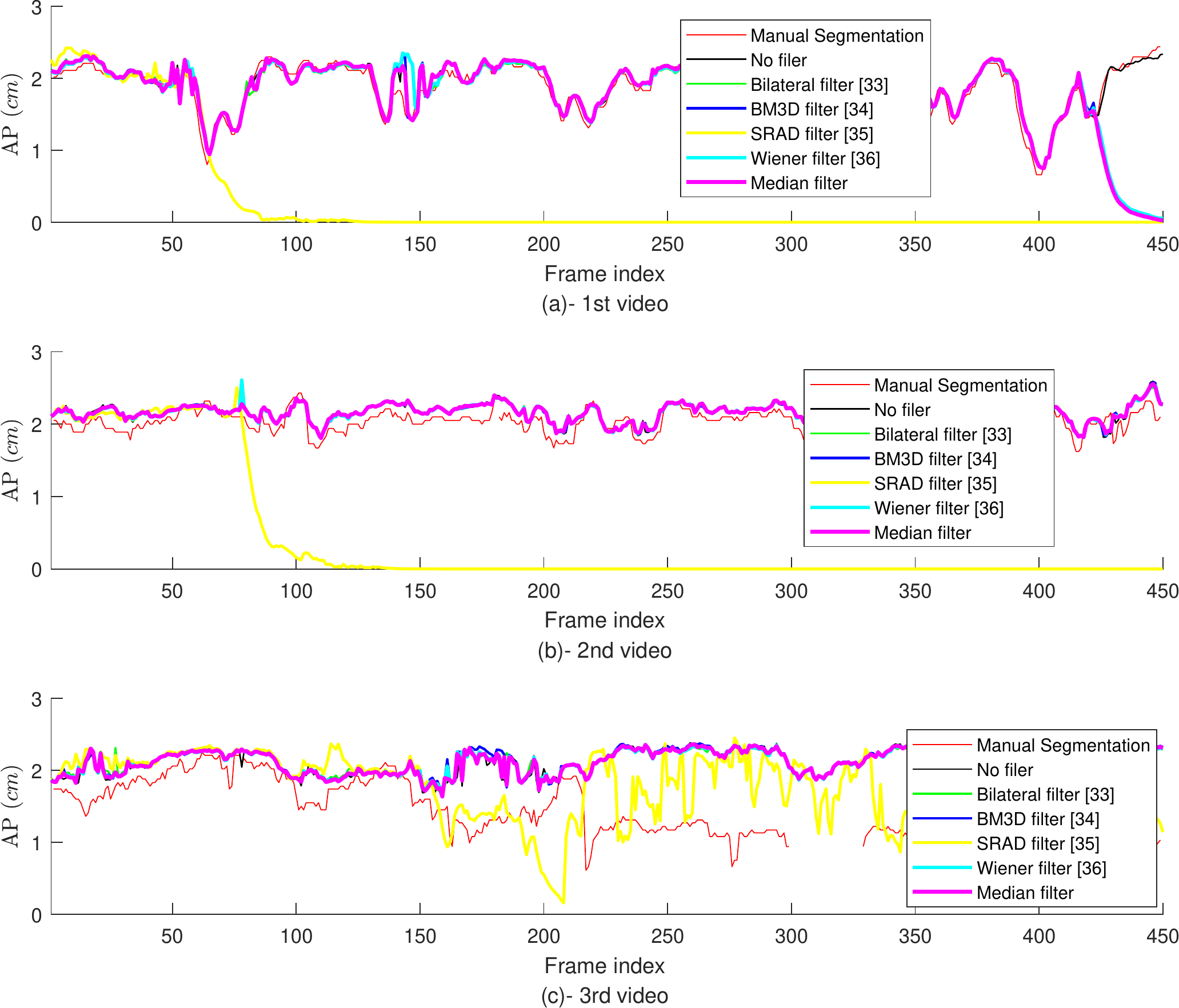}
    \caption{Estimated AP-diameter versus the type of speckle removal filter.}
    \label{fig:filter}
\end{figure}

\begin{table}[H]
\caption{RMS of error obtain from the proposed algorithm with different speckle removal filters.}
\begin{center}
 \begin{tabular}{||c|c|c|c|c|c|c|c|c||} 
 \hline
 \backslashbox{Metric}{Subject no.} 
                     & 1    & 2    & 3    & 4    & 5    & 6    & 7    & 8 \\ [0.5ex] 
 \hline\hline
 No. filter          & 0.10 & 0.16 & 0.26 & 0.33 & 0.25 & 0.25 & 0.10 & 0.12\\ 
   \hline
Bilteral filter [33] & 0.11 & 0.16 & 0.28 & 0.32 & 0.24 & 0.24 & 0.10 & 0.11 \\
  \hline
BM3D [34]            & 0.48 & 0.16 & 0.27 & 0.41 & 0.68 & 0.26 & 0.10 & 0.11 \\
  \hline
SRAD [35]            & 1.78 & 1.81 & 0.34 & 0.50 & 0.78 & 0.73 & 0.10 & 0.11 \\
 \hline
Wiener filter [36]   & 0.48 & 0.17 & 0.26 & 0.34 & 0.43 & 0.25 & 0.10 & 0.12 \\
 \hline
Median filter        & 0.48 & 0.17 & 0.26 & 0.36 & 0.53 & 0.25 & 0.10 & 0.12 \\
\hline
\end{tabular}
\end{center}
\end{table}
\vspace{-0.2cm}
\section{Conclusion}
In this paper, a novel active circle algorithm is developed for estimating the AP-diameter of the IVC in ultrasound imagery. It has been shown that the diameter of a circle fitted inside the IVC can accurately model the AP-diameter and therefore, is a useful tool capable of supporting further clinical research of using the IVC to guide fluid resuscitation. Furthermore, a novel evolution functional has been proposed and used for updating the circle parameters. Experimental results suggest that the proposed algorithm performs very close to the expert manual measurements. The proposed algorithm only failed under extremely low image-quality scenarios when the AP-diameter was even unable to be measured by the expert.
\biboptions{numbers,sort&compress}
\bibliographystyle{elsarticle-num}

\end{document}